\documentclass{article}
\pdfoutput=1

    \usepackage[preprint]{neurips_2022}

\usepackage[utf8]{inputenc} 
\usepackage[T1]{fontenc}    
\usepackage{url}            
\usepackage{booktabs}       
\usepackage{amsfonts}       
\usepackage{nicefrac}       
\usepackage{microtype}      
\usepackage{xcolor}         
\usepackage{verbatim}
\usepackage[ruled,linesnumbered]{algorithm2e} 
\usepackage{amsthm}
\usepackage {bbm}
\usepackage{titlesec}
\usepackage{enumitem}
\usepackage{subfigure}
\usepackage{tabularx}
\usepackage{amsmath,amssymb,amsfonts}
\usepackage{graphicx}
\usepackage{makecell}
\usepackage{multirow}

\usepackage{url}
\usepackage{natbib}

\newtheorem{theorem}{Theorem}

\newtheorem{lemma}[theorem]{Lemma}

\title{\texttt{\textsc{FadMan}}: Federated Anomaly Detection across Multiple Attributed Networks}

\author{%
     Nannan Wu \\
  Tianjin University\\
  China\\
  \texttt{nannan.wu@tju.edu.cn} \\
   \And
   Ning Zhang \\
   Tianjin University \\
   China\\
   \texttt{2020244106@tju.edu.cn} \\
   \AND
   Wenjun Wang \\
   Tianjin University \\
   China\\
   \texttt{wjwang@tju.edu.cn} \\
   \And
   Lixin Fan \\
   WeBank \\
   China\\
   \texttt{lixinfan@webank.com} \\
   \And
   Qiang Yang \\
   WeBanky \\
   China\\
   \texttt{qyang@cse.ust.hk} \\
}

\begin{document}

\maketitle
\vspace{-6mm}

\begin{abstract}
 Anomaly subgraph detection has been widely used in various applications, ranging from cyber attack in computer networks to malicious activities in social networks.
 Despite an increasing need for federated anomaly detection across multiple attributed networks, only a limited number of approaches are available for this problem. 
 Federated anomaly detection faces two major challenges.
 One is that isolated data in most industries are restricted share with others for data privacy and security.  
 The other is most of the centralized approaches training based on data integration. 
 The main idea of federated anomaly detection is aligning private anomalies from local data owners on the public anomalies from the attributed network in the server through public anomalies to federate local anomalies.
 In each private attributed network, the detected anomaly subgraph is aligned with an anomaly subgraph in the public attributed network.
 The significant public anomaly subgraphs are selected for federated private anomalies while preventing local private data leakage.
 The proposed algorithm \texttt{\textsc{FadMan}} is a vertical federated learning framework for  public node aligned with many private nodes of different features, and is validated on two tasks --- correlated anomaly detection on multiple attributed networks and anomaly detection on an attributeless network --- using five real-world datasets.
  In the first scenario,  \texttt{\textsc{FadMan}} outperforms competitive methods by at least 12\%  accuracy at 10\% noise level.
 In the second scenario, by analyzing the distribution of abnormal nodes, we find that the nodes of traffic anomalies are associated with the event of \textit{postgraduate entrance examination} on the same day.
\end{abstract}

\section{Introduction}
\begin{figure}[ht]
\setlength{\abovecaptionskip}{0.1cm}
\setlength{\belowcaptionskip}{-0.1cm}
	\centering
	\includegraphics[width=14cm,height=7.5cm]{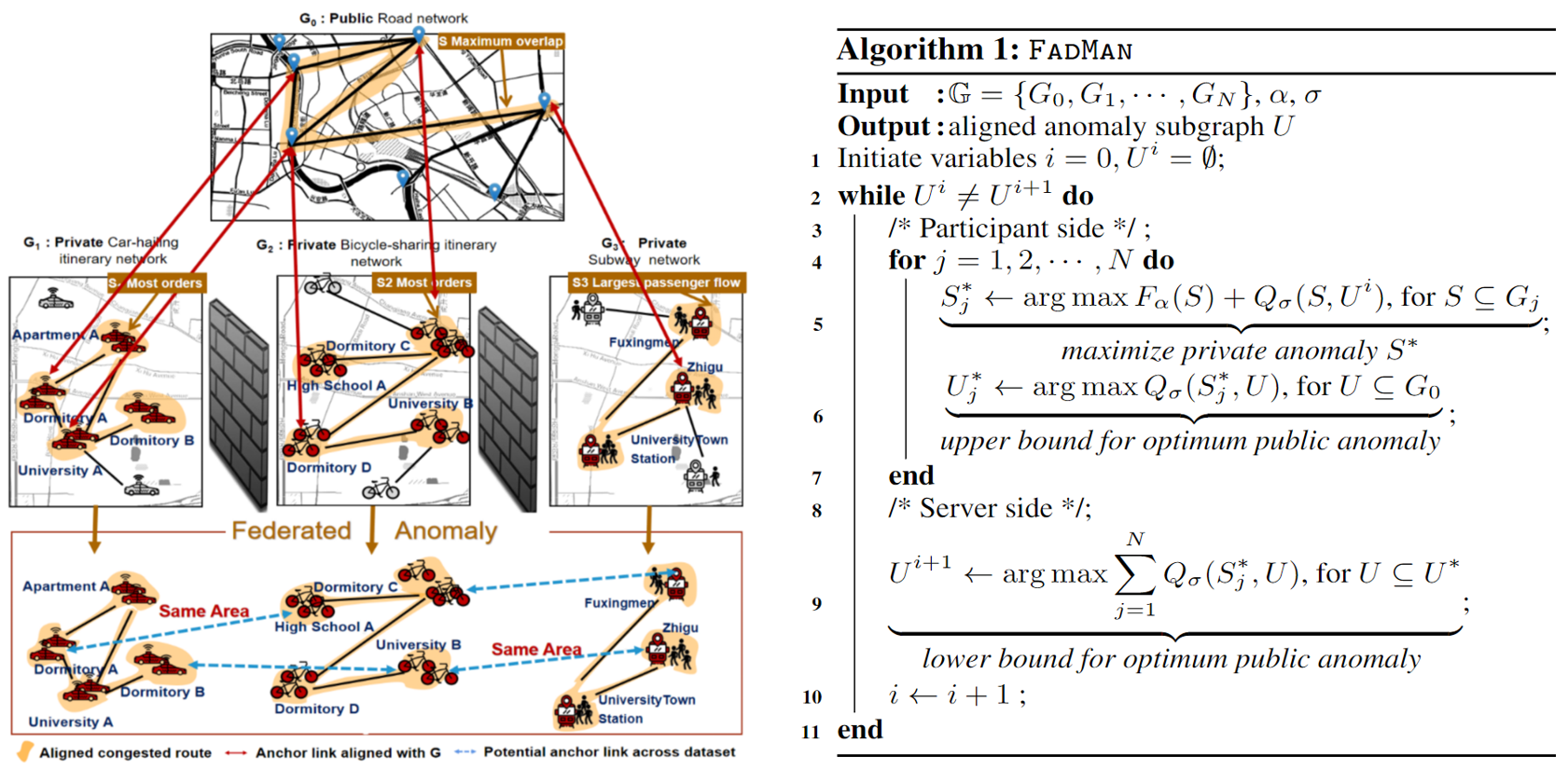}
	\vspace{-5mm}
    \caption{
    \textbf{Federated anomaly passenger flow} itineraries with similar structures in three private \textbf{industry traffic datasets} (Car-hailing \& Bicycle-sharing \& Subway). The federated anomaly is highly correlated (blue lines) in geographic attributes and specific locations (yellow shades).}
    \label{fig:example}
    \vspace{-5mm}
\end{figure}

The anomaly detection problem  has recently attracted much more attention. Many machine learning methods have been proposed to spot anomalies in different scenarios, such as disease outbreak detection in health alert networks, traffic jam detection in road networks, and event detection in social networks~\citep{kulldorff1997spatial,speakman2015scalable,speakman2013dynamic,takahashi2008flexibly,wu2018nonparametric,chen2019anomaly,parkinson2018graphbad}.
Most of the current methods are just used to centralized data.
However, most industry data exist in the form of isolated islands~\citep{YangLCT19}. 
Few solutions are proposed to the federated anomaly detection problems that are applied among multiple attributed networks without compromising user privacy~\footnote{In this paper, we use `graph' and `network', `vertex' and `node', and `attribute' and `feature’ interchangeably.
\vspace{-9.7mm}}.
Recently, anomaly detection methods are proposed to multiple networks without considering data privacy~\citep{bindu2017discovering,szmit2012usage}. 
These methods are focused on anomalies with the same anomaly characteristics on multiple attributed networks in a specific field, and  the accuracy improvement of the overall anomaly detection~\citep{Ying2020anomaly,futamura2020spectral}.
The work~\citep{bindu2017discovering} focuses on multi-layer social networks and uses hierarchical information to assist the overall abnormal ranking of nodes, which lacks analyses of the correlation  abnormalities among different social networks.

Attributed graph becomes more popular in industry data, which exhibit graph structure and node feature heterogeneity. 
We find that different graph anomalies, even from the same industry or same dataset, are non-IID on both structures and node features. Most industries protect the heterogeneity for the private reasons, e.g., a car-hailing company protects its travel demand network heterogeneity. It is a challenge to find the anomaly associations on multiple private attributed networks,  
and few further explorations have been carried out on multiple private attributed networks for the federated anomaly detection~\citep{zhao2019multi,chen2019network,ying2021network,zhao2020network}. 
In Figure~\ref{fig:example}, our proposed \texttt{\textsc{FadMan}} pays attention to federated anomalies of multiple private attributed networks and their potential correlations in various fields without compromising data privacy.
We can observe that the ``\textit{the largest passenger flow}'' subway station network aligning with  ``\textit{the most car-hailing and bicycle-sharing orders}'', which means that most people arrive or leave at a subway station with transfer to bicycle or car at the certain relevant location. 
In $n$ industries, there are multiple private attributed networks, $\{G_1,\cdots,G_N\}$. A well-known cost function $F_{\alpha}(S)$ for $S \subseteq G_i$ can be specifically modeled in the $i$th industry data to detect the anomaly. We select a public attributed network $G_0$, and present an alignment function $Q_\sigma(S, U)$ for $U \subseteq G_0$ to align the private anomaly $S$ on the public anomaly $U$. 
$F$ is the abnormal score of $S$ (e.g., the work Tree-Shape-Priors Subgraph Detection (TSPSD)~\citep{wu2018nonparametric} as $F$). 
$Q$ is the alignment score of $S$ and $U$ (e.g., the work  Cross-Network Embedding for Multi-Network Alignment (CrossMNA)~\citep{chu2019cross} as $Q$). 
The parameters $\alpha$ and $\sigma$ are significant level and alignment threshold respectively.

\textbf{Related work.} Anomaly detection has been extensively studied from outlier detection, anomaly subgraph detection in one attributed network and multiple attributed networks to federated anomaly detection in multiple isolated attributed networks. Most of these algorithms are used in a single attributed network~\citep{cadena2018graph,bhatia2020midas,cai2020structural,boniol2020series2graph}. A few algorithms are used in multi-layer networks~\citep{bindu2017discovering,de2015ranking,tam2019anomaly,jie2020framework} for abnormal nodes. 
The methods based on network alignment are used to multi-layer networks for anomaly subgraph detection~\citep{Ying2020anomaly,zhang2015cosnet,liu2014hydra,heimann2018regal,ye2019vectorized,huynh2019network,yan2021bright,sun2019dna,wang2018deepmatching,meng2019deep}. Many federated learning methods on graphs are proposed to explore graph embedding vectors average~\citep{xie2021federated,li2020federated,li2021ditto,zhang2021subgraph}. However, the transmitted gradient parameters are still at risk of data leakage~\citep{NEURIPS2019_60a6c400,lyu2020threats,melis2019exploiting}. There are approaches that tackle statistical heterogeneity by considering local partial data or some server-side proxy data as model parameters~\citep{li2020federatedc,xie2021federated}. We are motivated on private anomalies aligning on public anomaly  to protect data privacy and federate local anomalies.
We summarize our main contributions as follows:
\vspace{-1mm}
\begin{itemize}[leftmargin=*]
\item {\bf Innovative work.} A novel pioneer algorithm is first proposed to federated anomaly detection on multiple attributed networks in the vertical federated learning  without compromising data privacy. 
\item {\bf Theoretical guarantee.} The proposed algorithm enjoys a proved convergence property. A broad spectrum of anomaly detection functions is applicable under satisfying the function properties. Specifically, for the graph scan statistics (e.g., Berk-Jones), the algorithm succeeds in federated anomaly detection with car-hailing, bicycle-sharing and subway industry data.
\item {\bf Effectiveness and robustness.} Extensive experiments on real datasets have verified that \texttt{\textsc{FadMan}} can be effectively applied to different multi-network scenarios. 
On the real computer network dataset, our method achieves 97\% accuracy at the ten percent noise level, which is 12\% higher than the  competitive methods in accuracy.
\end{itemize}

\section{Problem formulation}
A data owner has an attributed network $G = (V, E, P)$, $V=\{v_1,\cdots,v_n\}$,  $E\subseteq V\times V$ are sets of vertices and edges in $G$, and $P \in \mathbb{R}^n$ is the specific anomaly feature set of $G$. $P$ is widely obtained by the mapping function $\mathbf{p}: V \rightarrow[0, 1]$ defines the empirical p-value corresponding to each node $v \in V$ \citep{wu2018nonparametric,chen2014non}, the smaller the p-value, the more abnormal the node.
Anomaly target detection can be represented as identifying subgraphs $S\subseteq G$  whose vertex set $V_{S}$ and edge set $E_{S}$ are subset of $V$, $E$.  

We consider a set of data owners denoted by $\{G_1,G_2,\ldots, G_N\}$ as multiple attributed networks. One global anomaly $\{S_1,S_2,\ldots, S_N\}$ (e.g., anomaly transportation in a city) is distributed as local anomalies on multiple attributed networks. We denote $[N]$ as the set of $\{1,2,\cdots,N\}$. The local anomaly $S_i$ for $i  \in [N]$ can be identified by maximizing the local model $F_\alpha^i(S_i)$ with the significant level $\alpha$. 
We introduce a public data owner $G_0$ to bridge the  gap among isolated private data owners.
A public anomaly  $U \subseteq G_0$ instead of the global anomaly is assoicated with local anomalies. An local alignment function $Q_\sigma^i(S_i, U)$ is employed to measure the similarity between $S_i$ and $U$ for $i \in [N]$.
For anomaly subgraphs $S_i\subseteq G_i$ and $U \subseteq G_0$, each private data owner exclusively defines the alignment matrix $A^i \in \{0,1\}^{n \times n}$, where $A_{uw}^i \leftarrow 1$ if the value of $(u,w)$ is greater than the alignment threshold $\sigma$ (i.e., a constant) with $u \in V_{S_i}$ and $w \in V_{U}$, $A_{uw}^i \leftarrow 0$ otherwise.

\textbf{Federated anomaly detection problem.}
The federated anomaly detection across multiple attributed networks can be formulated as follows:
\begin{equation}
\label{L}
(S^*, U^*) = \arg\max_{U \subseteq G_0} \sum_{i=1}^N \max_{S_i \subseteq G_i} F_{\alpha}^i(S_i) + Q_{\sigma}^i(S_i, U)
\end{equation}
where the private anomaly subgraphs $S^*$ and the public anomaly subgraph $U^*$ are the optimal solutions. The private anomaly subgraph $S_i \subseteq G_i$ is measured by $F$ with the significant level $\alpha$ (e.g., 0.15). The alignment score between $S_i$ and the public anomaly $U$ is measured by $Q$ with the predefined alignment threshold $\sigma$ (e.g., 0.8).
For the optimization problem~(\ref{L}), we can assume that the desired anomalies $\{S_i\}$, $U$, and the functions $F_\alpha^i$, $Q_\sigma^i$, intuitively satisfy  the  four  properties:

\noindent \textbf{     (P1)} $F_\alpha^i$ is monotonically \textbf{\textit{increased}} with the number of abnormal nodes in $S_i$;

\vspace{-1mm}
\noindent \textbf{     (P2)} $F_\alpha^i$ is monotonically \textbf{\textit{decreased}} with the number of normal nodes in $S_i$;
    
\vspace{-1mm}
\noindent \textbf{     (P3)} $Q_\sigma^i$ is monotonically \textbf{\textit{increased}} with the number of node pairs aligned between $S_i$ and $U$;

\vspace{-1mm}
\noindent \textbf{     (P4)} $Q_\sigma^i$ is monotonically \textbf{\textit{decreased}} with the number of nodes from $S_i$ and $U$.

These properties  P1-P2 are widely used in connected anomaly subgraph detection~\citep{chen2014non,wu2018nonparametric}, and the possibility of same anomaly is increased with aligned links P3~\citep{Ying2020anomaly}. 
We can observe that when the anomaly $U$ is the whole graph $G_0$, there are maximum links between $S_i$ and $U$  by the property P3. The anomaly $U$ is not desired for including more noise (i.e., normal nodes). Similarly, we have $S_i$. In this work, we include the property P4, which restricts on the smaller sizes of anomalies $S_i$ and $U$.

\section{Proposed algorithm}
The federated anomaly detection problem~(\ref{L}) does not explicitly allow for data protection in different data owners. We propose the algorithm that is tractable to protect data and  detect federated anomaly.
The key idea of proposed algorithm is to aggregate local public anomalies and then re-distribute the optimized public anomaly back to each local data owner. 

\subsection{Federated anomaly detection}
The main steps of the  \texttt{\textsc{FadMan}} in Algorithm 1 are described as follows:

For the $i$-th participant side (local data owner), given the optimized public anomaly $U$,




(c1) $S_i^* \leftarrow \arg\max_{S \subseteq G_i} F_\alpha^i(S) + Q_\sigma^i(S, U)$

(c2) $U_i^* \leftarrow \arg\max_{U \subseteq G_0} Q_\sigma^i(S_i^*, U)$

For the server side, aggregate local public optimized anomalies $\{U_i^*\}$ for $i \in [N]$,

(s1) Sort the local data owner $U_{(i)}^*$ in ascending order of size $|V_{U_{(i)}^*}|$

(s2) For a coalition $C=\{\}$, start at the owner $i$ joining the coalition $C$, and stop when the first owner  would increase the alignment error  $(\sum_{k \in C} |V_{U^*_k}|-|V_{U^*_{(i)}}|)/(|V_{U^*_{(i)}}| \sum_{k \in C} |V_{U^*_k}|)  Q_\sigma^i(S_{(i)}^*, U_{(i)}^*)$ by joining the coalition $C$. Then, the optimal partition of $[N]$ is made up of the coalition set $\{C\}$. 

(s3) Select a coalition $C$ with the minimum score, and set the new public anomaly $U=\bigcup_{i \in C} U_{(i)}^*$

We can observe that the private data $\{S_i\}$ at the local data owners would not be uploaded to the server. At the server side, we can just obtain the public anomalies $\{U_i^*\}$ and its alignment scores.
The problem is solved with two different settings: \textit{participant side} for private anomaly detection has $N$ data owners who do not share data with each other; and \textit{server side} for public anomaly alignment has the public attributed network which is shared by the data owners. 
Algorithm 1 is iterated between the participant side and the server side, until the stop condition is  satisfied (i.e., the public anomaly subgraph $U$ does not change).
\subsubsection{Private anomaly detection.}
Each data owner employs a function $F$ to detect self-defined anomaly by consolidating their private attributed network data. The self-defined anomaly at step 5 and step 6 of Algorithm 1 is detected exclusively in the data owner setting. The anomaly detection provides meaningful privacy guarantees.
In this paper, we employ the non-parametric graph scanning statistic $F$ as a score function, and its form is defined as follows:
\begin{equation}
\label{funobj}
    F_\alpha(S) = \varphi(\alpha,N_\alpha(S),N(S)). 
\end{equation}
where $S$ is a set of connected vertices, that is, a subgraph, and $\alpha$ is the significant level,  $N_\alpha(S)$ is the number of anomaly vertices in $S$ whose p-value is less than or equal to $\alpha$, and $N(S)$ is the total number of vertices in $S$.
In this paper, we consider two non-parametric graph scanning statistics as the score function~(\ref{funobj}): Berk-Jones (BJ) statistic \citep{bj} and Higher Criticism (HC) statistic \citep{hc}.
We use non-parametric graph scanning statistics to measure the subgraph abnormality as a numerical value. The optimal anomaly subgraphs maximize the graph scanning statistics $F$ over each private network.

\subsubsection{Public anomaly alignment.}
Given the public attributed network $G_0$ and another attributed network $G_i$ for $i \in \{0,1,\cdots,N\}$, we define the function $Q$ as measuring the anomaly alignment score between $S \subseteq G_i$ and $U \subseteq G_0$:
\begin{equation}
    Q_{\sigma}(S, U) = \frac{N_\sigma(S, U)}{N(S)} +  \frac{N_\sigma(S, U)}{N(U)}
\end{equation}
where $\sigma$ is the predefined alignment threshold. The anomaly subgraph $S$ is aligned on the public anomaly subgraph $U$. The two subgraphs are connected. 
$N_\sigma(S, U)$ is the number of node pairs between $S$ and $U$ whose alignment probability is greater than or equal to $\sigma$. $N(S)$ and $N(U)$ are the number of all nodes in $S$ and $U$ respectively.
The node's alignment probability is obtained through the network alignment work CrossMNA \citep{chu2019cross}. By introducing this algorithm, we pre-aligned each private network with the public network, and obtained the alignment probability of all pairs of nodes between them. We use network alignment to map the similarity between the subgraphs to a value, and obtain the most similar part between  two subgraphs by maximizing $Q$ at step 6, 9 of Algorithm 1.

\subsection{Theoretical analysis}


We now analyze the increasing property in the objective when one iteration of the proposed algorithm is performed. The convergence of \texttt{\textsc{FadMan}} is proved based on several lemmas and the four properties of objective function, without the specific forms of $F$ and $Q$ functions. Each lemma is a building-block that describes that there exists an local optimal solution.
\begin{theorem}[Convergence and Optimality of \texttt{\textsc{FadMan}}]
\label{theorem:main}
Within the pre-aligned domain (each data owner integrates the alignment probability information of all node pairs between its private network and the public network), given the parameter $\alpha$ and $\sigma$ settings,  Algorithm 1 converges to the optimal solution of the problem~(\ref{L}). 
\end{theorem}
Though Theorem~\ref{theorem:main} looks straightforward to a network-based algorithm framework for federated anomaly detection, the problems at steps (c1) and (c2) can be addressed with the previous methods~\cite{Ying2020anomaly,wu2018nonparametric} support the optimal solutions. At the step (c1), given the public anomaly $U$, the optimal private anomaly $S_i^*$ is obtained. We are target for the new improved public anomaly $U$ at the next iteration. 

At the $t$-th iteration, for the server side, we can observe the public anomalies $\{U_i^*\}$ and the function values of $F$ and $Q$. We assume that there exists a public anomaly subgraph $\underline{U^* \subseteq G_0}$ maximizing $\sum_{i \in [N]} Q_\sigma^i(S_i^*, U^*)$ under the fixed $S_i^*$. The value of $F_\alpha^i(S_i^*)$ will not change with the update $U$.
\begin{lemma}[$Q$ upper bound]
We must have the upper bound of the optimal $Q$ under the fixed $S_i^*$.
\begin{equation}
\sum_{i=1}^N   Q_\sigma^i(S_i^*, U_i^*)
 \geq 
 \sum_{i=1}^N  Q_\sigma^i(S_i^*, U^*).
\end{equation}
\label{lmup}
\end{lemma}
\begin{lemma}[$Q$ lower bound]
A public anomaly $\underline{U^o \subseteq \bigcup_{i \in [N]} U_i^*}$ maximizes $\sum_{i \in [N]} Q_\sigma^i(S_i^*, U^o)$. We must have the lower bound of the optimal $Q$ under the fixed $S_i^*$.
\begin{equation}
 \sum_{i=1}^N   Q_\sigma^i(S_i^*, U^*)
 \geq 
 \sum_{i=1}^N   Q_\sigma^i(S_i^*, U^o).
 \label{eq:lower}
\end{equation}
\label{lmlow}
\end{lemma}
\begin{lemma}[$Q$ optimum solution]
The solution $U^o = U^*$ maximizes $Q$ under the fixed $S_i^*$ over $G_0$. The optimum solution must exist between
\begin{equation}
    \min_{U \subseteq \bigcup_{i \in [N]} U_i^*} \sum_{i=1}^N   \Big[Q_\sigma^i(S_i^*, U_i^*) -   Q_\sigma^i(S_i^*, U)\Big]
    \label{eqopt}
\end{equation}
\label{lmopt}
\end{lemma}
The public anomaly $U_i^*$ is the local maximum solution. The $Q$ upper bound is derived intuitively from summing local maximum solutions is greater than the global maximum solution.
For the $Q$ lower bound, the maximum solution over the local set is less than or equal the maximum solution over the global set.
We can observe that the maximum solution must exist between the upper bound and lower bound, and minimize the gap between the two bounds. Note that when the gap is 0, the optimum solution is equal to the solution of lower bound. All of the local solutions $U^*_i$ are the same.

However, in the problem~(\ref{eqopt}), $Q_\sigma^i(S_i^*, U)$ can not be computed at the server because the forms of function $Q$ and private anomalies $\{S_i^*\}$ are not accessed for the data protection rules. We employ the framework for optimality and stability in federated learning to address the problem~(\ref{eqopt})~\cite{donahue2021optimality}. We define a coalition $C \subseteq [N]$ is a subset of local data owners, and all owners $[N]$ are partitioned into coalition set $\{C\}$. We consider the gap between two bounds~(\ref{eqopt}) as errors across owners. By the properties P3 and P4, we know that the value of $Q_\sigma^i(S_i^*, U)$ changed with the size of $U$ for the fixed $S_i^*$. Given a partition $\Pi$ of $[N]$, we define its error below:
\begin{equation}
  f_{\{U_i^*\}}(\Pi) = \sum_{C \in \prod} \sum_{i \in C} |V_{U_i^*}| \cdot \bigg(\frac{\sum_{k \in C} |V_{U^*_k}|-|V_{U^*_i}|}{|V_{U^*_i}|\cdot \sum_{k \in C} |V_{U^*_k}|}\bigg)  Q_\sigma^i(S_i^*, U_i^*)
  \label{eq:ferror}
\end{equation}
\begin{lemma}[Theorem 1, from~\cite{donahue2021optimality}]
  An optimal partition $\Pi$ can be created for a set of local data owners $[N]$. First, start with every owner maximizing the local public anomaly $U_i^*$. Then, group the owners together in ascending order of node size, halting at the first owner would increase its error $(\sum_{k \in C} |V_{U^*_k}|-|V_{U^*_i}|)/(|V_{U^*_i}| \sum_{k \in C} |V_{U^*_k}|)  Q_\sigma^i(S_i^*, U_i^*)$ by joining the coalition $C$. The resulting partition $\Pi$ is optimal.
\label{lmthm}
\end{lemma}
By Lemma~\ref{lmthm}, we select a coalition $C \in \Pi$ with the minimum error, and achieve the desired public anomaly $U=\bigcup_{i \in C} U_{i}^*$. The public anomaly $U$ is the optimal solution for $\sum_{i\in [N]}Q_\sigma^i(S_i^*,\cdot)$.

Note that the bi-level optimization problem~(\ref{L}) has two targets. The local function $F_\alpha^i(S_i)$ aims to detect private anomalies using only the data of local data owner $i$. The alignment function $Q_\sigma^i(S_i^*, U_i^*)$ aims to federate these anomalies by aligning the local private anomaly on the public global anomaly. We can observe that the two functions exhibit at two scales, and introduce a hyperparameter $\lambda$, $\max_{U \subseteq G_0} \sum_{i \in [N]} \max_{S_i \subseteq G_i} F_{\alpha}^i(S_i) + \lambda Q_{\sigma}^i(S_i, U)$, controls the interpolation between the two functions. The hyperparameter $\lambda$ is related to the specific form of $Q$, e.g., $Q$ as a regularization term, and the theoretical properties of \texttt{\textsc{FadMan}} is proved under the general setting without $\lambda$. When $\lambda$ is set to 0, our algorithm performs as local anomaly detection tasks for each local data owner. Our algorithm focuses on more federated anomalies with increasing $\lambda$.

\textbf{(Relation to Ditto federated learning}~\cite{li2021ditto}\textbf{)}. We do not consider the network structure $\{G_i\}$, and each $G_i$ can be transformed to a vector $v_i$, e.g., $v_i(k) \leftarrow 1$ if the node $k$ is abnormal in $G_i$, and $v_i(k) \leftarrow 0$ otherwise. The vector $w$ is derived from $G_0$. We can take $Q$ as a regularization term, $\lambda/2 \cdot \parallel Av_i - w \parallel^2$. The optimization problem is to minimize $-F_i(v_i)+\lambda/2 \cdot \parallel Av_i - w^* \parallel^2$ for $v_i$ with each owner, subject to $w^* = \arg\min_{w} 1/N \cdot \sum_{i \in [N]} \parallel Av_i - w \parallel^2$.  Our  algorithm can be reduced to Ditto. Thus we propose a general algorithm framework for federated anomaly detection.

\begin{figure*}[t]
\setlength{\abovecaptionskip}{0.1cm}
\setlength{\belowcaptionskip}{-0.3cm}
    \centering 
    \includegraphics[width=0.95\linewidth]{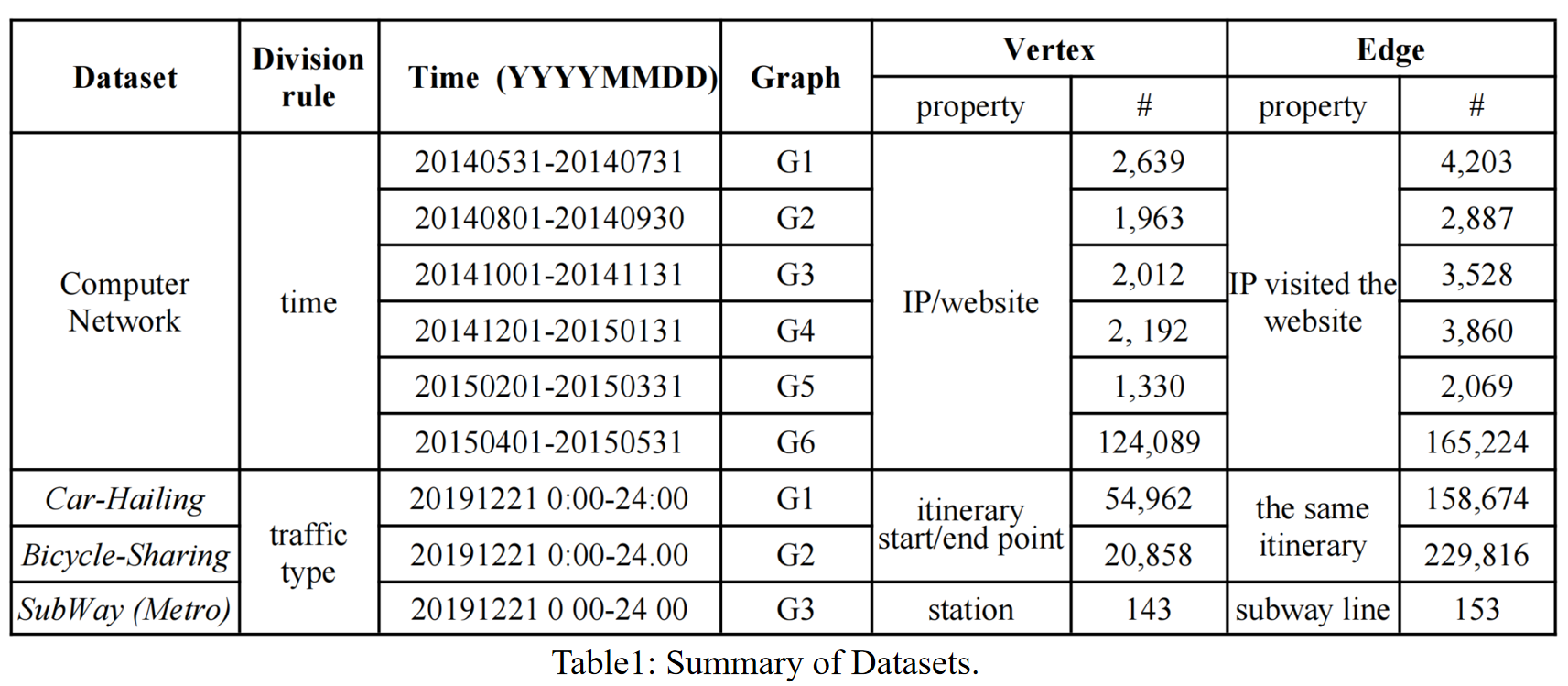}
    \label{fig:tableab}
    \vspace{-3mm}
\end{figure*}

\begin{figure}[t]
    \centering 
    \includegraphics[width=1\linewidth]{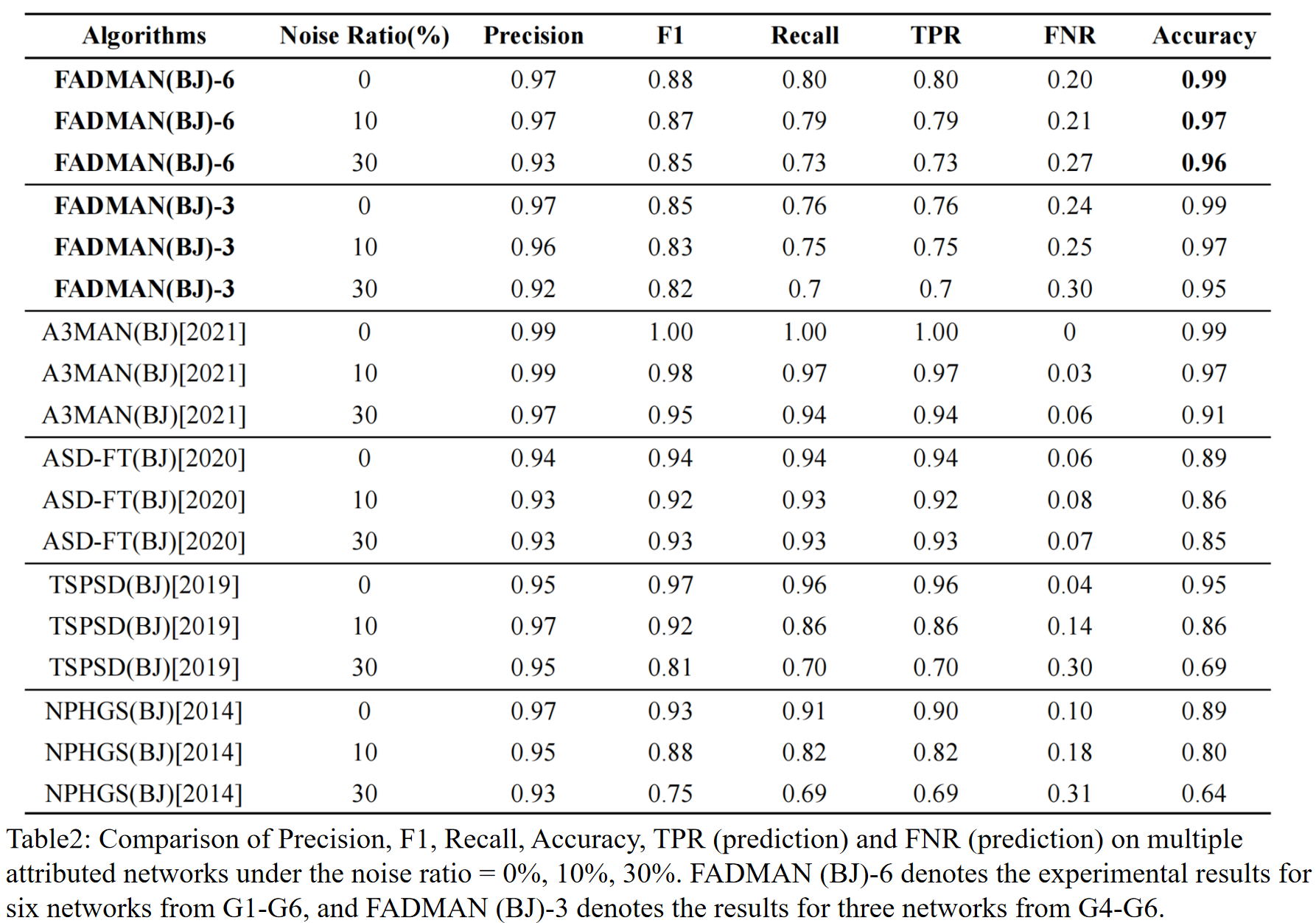}
    \label{fig:table2}
    \vspace{-0.8cm}
\end{figure}

\begin{figure}[t]
    \centering 
    \includegraphics[width=1\linewidth]{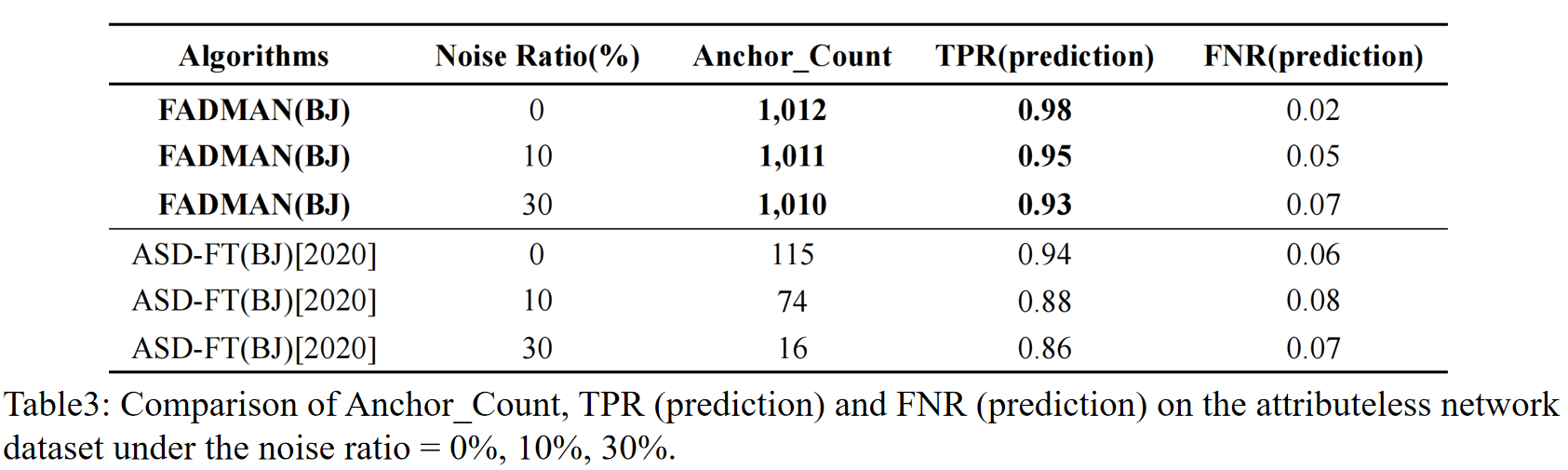}
    \label{fig:table3}
    \vspace{-0.5cm}
\end{figure}

By the theorem, our algorithm guarantees on detecting the most anomaly subgraphs on multiple private attributed networks. \texttt{\textsc{FadMan}}'s time complexity is \textit{$O(kN|V|^2)$}, where $k$ is the number of iterations, $N$ is the number of networks, $V$ is the number of nodes of the network. In practice, the implementations of $F_\alpha$ and $Q_\sigma$ are normalized because the function value ranges are different.

\begin{figure*}[t]
\setlength{\abovecaptionskip}{0.1cm}
\setlength{\belowcaptionskip}{-0.1cm}
    \centering 
    \includegraphics[width=14cm,height=8cm]{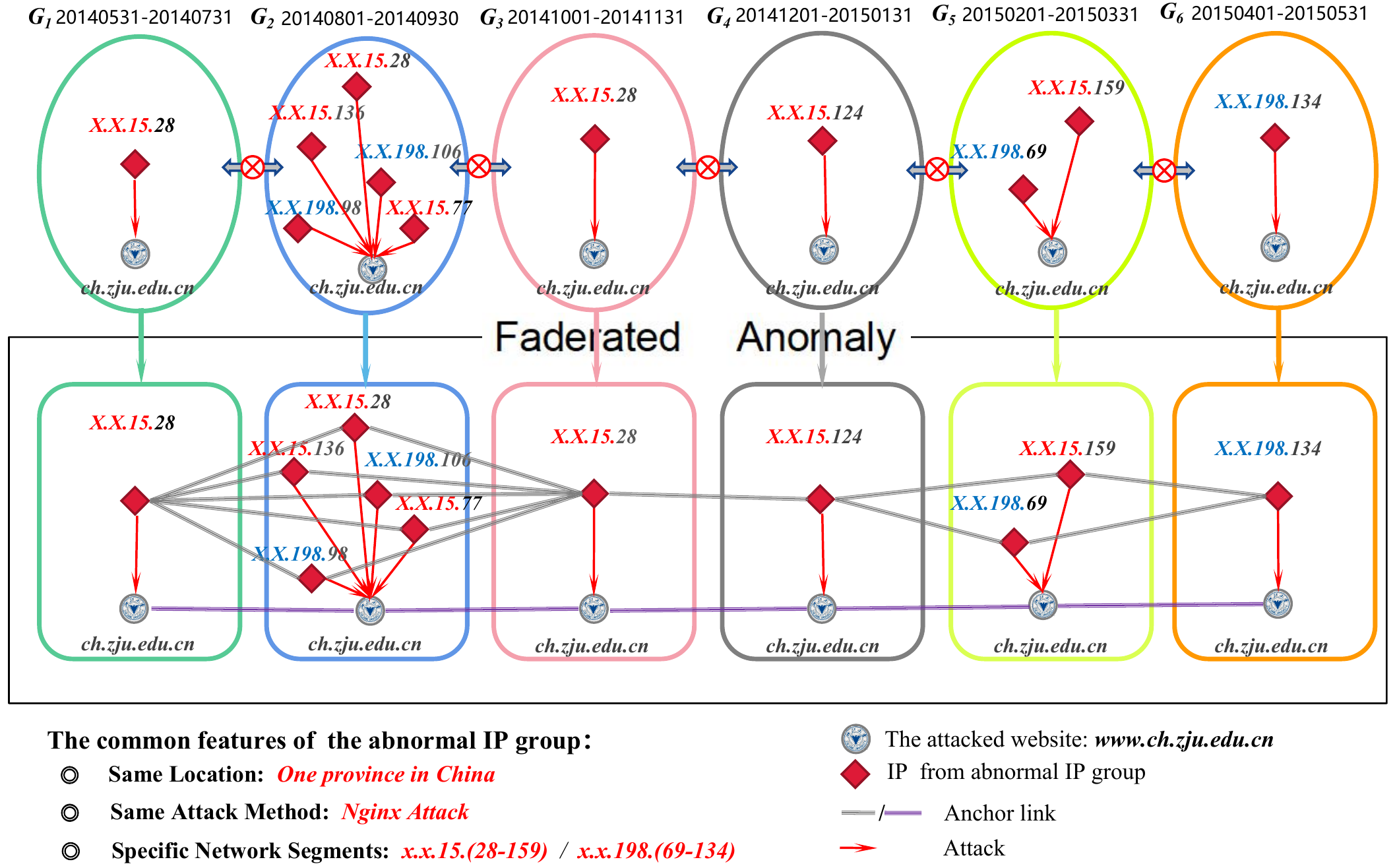}
    \caption{\textbf{A set of related abnormal IPs was detected by our method}. 
    The detected  IPs have potential correlations (e.g., at the same place). The site \textit {www.ch.zju.edu.cn} was attacked mainly  from two network segments \textit {x.x.15.(28-159)} and \textit {x.x.198.(69-134)}. The addresses of these IPs were all at the same place (i.e., \textit {Shanxi, China}). The attack methods were all \textit {Nginx Attack}.}
    \label{fig:case1}
\end{figure*}
\begin{figure*}[t]
\setlength{\abovecaptionskip}{0cm}
\setlength{\belowcaptionskip}{0.1cm}
    \centering 
    \includegraphics[width=0.95\textwidth]{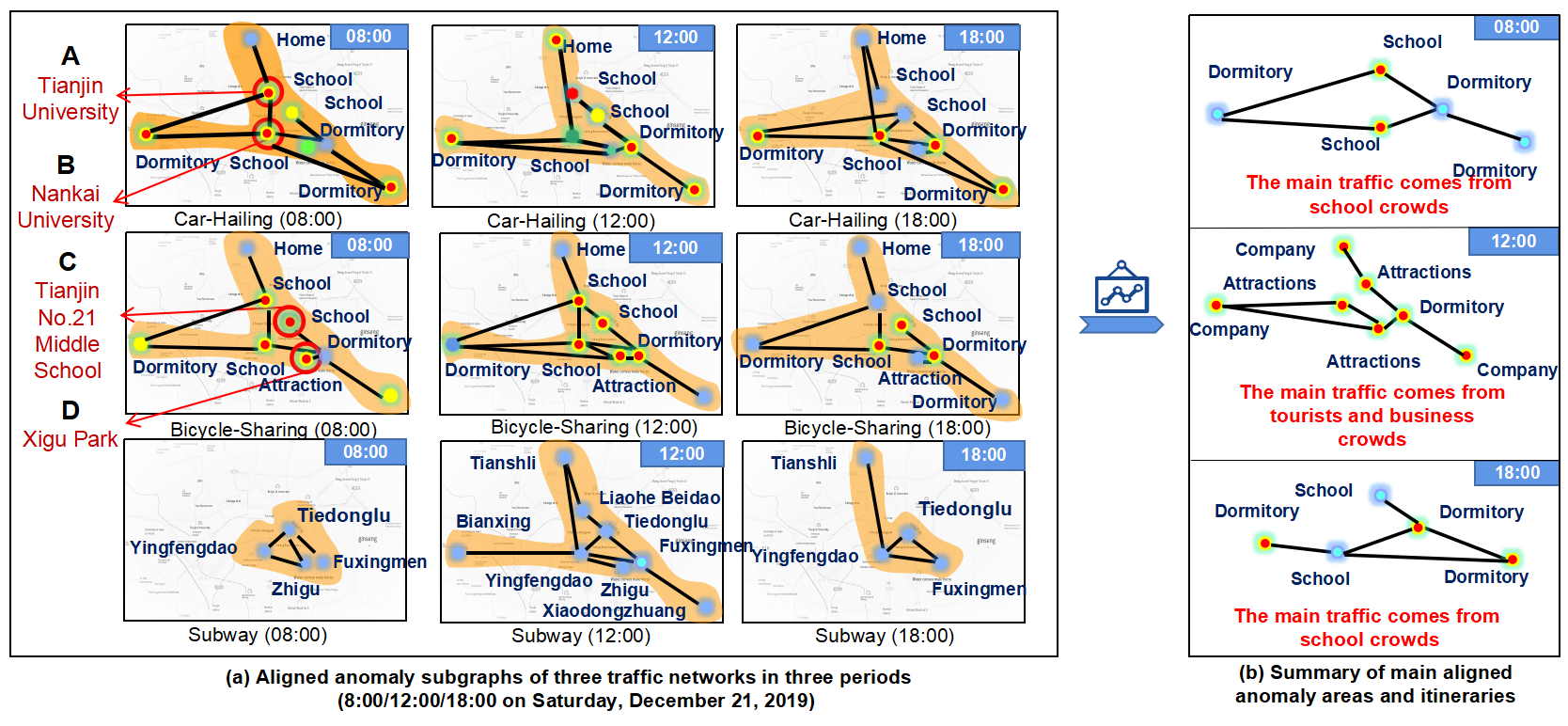}
    \caption{\textbf{Correlated abnormal distribution of Car-hailing $\&$ Bicycle-sharing $\&$ Subway detected by \texttt{\textsc{FadMan}}.} 
    We detected the aligned anomaly subgraphs at 8:00, 12:00 and 18:00 on December 21, 2019. (a) is a heat map of the aligned abnormal distribution of the traffic networks. Each hot spot represents an abnormal area, and the redder the color, the higher the abnormality of the area. The black lines indicate that there are itineraries between the two areas, the dark blue text indicates the main position type of the start and end points of the itineraries included in the area. (c) is a summary of the aligned abnormalities at three periods in (a). From the extracted main aligned anomaly areas and itineraries, we can observe that morning and evening traffic flow mainly comes from school crowds, while at noon, it is mainly business crowds and tourists.
   }    \label{fig:case2}
   \vspace{-6mm}
\end{figure*}
\section{EXPERIMENTS}
In this section, we performed a series of experiments to verify \texttt{\textsc{FadMan}}. We applied the algorithm to two real scenarios and verified the effectiveness of the algorithm on the two tasks of correlated anomaly detection on multiple attributed networks and anomaly detection on an attributeless network through ground truth. We compared it with four competitive baselines.

\textbf{Datasets.}
We constructed two scenarios (Table 1) based on five real datasets: Computer network, Car-hailing \& Bicycle-sharing \& Subway (Metro) traffic \& POI (Point of Interest) dataset. POI is the ground truth dataset to case study the federated anomaly association with interest locations. 

In the first scenario, we divided the computer network into 6 private networks (with anomalous properties) according to time, treated the original network as a public network (only topology information), and conducted two comparative experiments on it: correlated anomaly detection on multiple attributed networks and anomaly detection on the attributeless network (regard $G6$ as an attributeless network by setting the p-value of all nodes in $G6$ to 1). In the second scenario, we use Car-hailing itinerary network, Bicycle-sharing itinerary network, and Subway network as the 3 private attributed networks, regard POI dataset as the public network. In this scenario, we conducted an experiment: correlated anomaly detection on multiple attributed networks.

\textbf{Methods.} In this work, we use anomaly detection algorithms NPHGS \citep{chen2014non}, TSPSD\citep{wu2018nonparametric},A3MAN~\cite{zhang2018anomaly}, and anomaly alignment algorithm ASD-FT\cite{Ying2020anomaly} as baselines. 

\textbf{Metrics.}
We use Recall, Precision, F1 ,Accuracy, TPR (True Positive  Rate), FNR (False Negative Rate)
to evaluate algorithms' ability to detect the correlated anomaly on multiple attributed networks, and use
Anchor\_Count, TPR (prediction), FNR (prediction) to evaluate the algorithms' ability to detect anomalies on attributeless networks and discover related links among anomalies across the network (see supplementary in detail). 

\subsection{Experiment results}
We conduct comparative experiments on the computer network dataset with ground truth data, set $\alpha=0.15$, $\sigma=0.8$. 
The experimental results are shown in Table 2
and Table 3.

\textbf{1) Ability to detect correlated anomalies on multiple attributed networks:} 
Taking all the attributed computer networks as the input of \texttt{\textsc{FadMan}} and getting the results in Table 2. In this case, \texttt{\textsc{FadMan}} (BJ)-6 denotes the experimental results for six networks from G1-G6, and \texttt{\textsc{FadMan}} (BJ)-3 denotes the results for three networks from G4-G6. Except for the \texttt{\textsc{FadMan}} method, the other comparison methods are non-federated anomaly detection algorithms. From Table 2, we can draw the following conclusions:

Firstly, \texttt{\textsc{FadMan}} has reached 0.97 for Precision and 0.99 for Accuracy, which are basically the same as the non-federated methods, while protecting privacy. F1 and Recall both exceed the NPHGS method at a noise level of 30, and FNR is higher than the comparison method at all noise levels, thus showing that \texttt{\textsc{FadMan}} not only protects well privacy, while being able to have efficient anomaly detection.

Secondly, we conducted comparison experiments on three datasets and six datasets, respectively, and it can be obtained from Table 2 that each experimental metrics of the six datasets is higher than the experimental results of the three datasets, which is because considering more information is more beneficial to capture useful information and thus improve the performance of anomaly detection.


\textbf{2) Ability to detect anomalies on attributeless network:} 
By setting the p-values of all nodes in $G_6$ to $1$, it is regarded as an attributeless network. Then run \texttt{\textsc{FadMan}} to get metrics in Table 3 (Higher Criticism (HC) statistic results are the same as BJ). 
For comparable TPR (prediction) and FNR (prediction), the \texttt{\textsc{FadMan}} algorithm reached $0.98$ and $0.02$, which is significantly better than the baselines. Moreover, compared with the ASD-FT anomaly alignment algorithm, the total number of abnormal anchor links obtained by our algorithm is $1,012$, which is $8.8$ times than $115$ of ASD-FT, which proves the effectiveness of \texttt{\textsc{FadMan}}.


\subsection{Case study in computer network dataset}
Run \texttt{\textsc{FadMan}} on computer network dataset, input all networks, and set $\alpha=0.15$, $\sigma=0.8$. 

\textbf{1) Discovery of related abnormal IP group}: Our algorithm can obtain the abnormal IP group and mine the hidden attacking IP information (Figure \ref{fig:case1}). \texttt{\textsc{FadMan}} can mine potential abnormal anchor links among multi-layer networks. By summarizing the anchor nodes corresponding to these anchor links, we can obtain an abnormal IP group. Although these IPs appear in different periods, their attack behaviors are similar. Through their log information, we found these IPs come from several fixed network segments, and their attack methods and locations are also the same, which means that these IPs may come from the same attack source. Our method finds effective cyber attacks and thus improves the prevention of cyber attacks.

\textbf{2) Prediction of cyber attacks}: We treat $G_6$ as an attributeless network by setting the p-value of all nodes in $G_6$ to $1$ and use it with other networks as the input of \texttt{\textsc{FadMan}} to obtain its anomaly subgraph $S_6$. Regard $S_6$ as the prediction result, which summarizes the IPs that may attack the website during the period of $G_6$. We compare it with the real attacks that occurred during this period, and get the TPR (prediction) and FNR (prediction) in Table 3.
We can observe from the metrics that \texttt{\textsc{FadMan}} can make reasonably accurate predictions of future attacks. Our algorithm can detect the abnormal situation of the target network through networks with sufficient abnormal characteristics, even if the target network does not have any abnormal information.

\subsection{Case study in traffic datasets}

Run \texttt{\textsc{FadMan}} on multiple networks composed of the car-hailing itinerary network, the bicycle-sharing itinerary network, and the subway (metro) network. Match the abnormal detection results on the three traffic networks with the POI data to obtain the real information of the detection results, and then mapping it to a map. We selected three networks with the same period (8:00/12:00/18:00) and set $\alpha=0.05$, $\sigma=0.8$. We choose the illustrative analysis obtained on datasets of 8:00 as an example:


\textbf{Discovery of real events}: For the anomaly detection results of different datasets, we mapped them to the map after matching with POI data, and the obtained result graph is collectively referred to as the mapping graph below. From the mapping graph of a single dataset (Bicycle-Sharing), we can find that most of the anomalies are clustered near schools, hospitals and parks, and a small number of anomalies are also distributed around some large shopping malls. Schools, hospitals, and parks are popular places in a city in the morning. This distribution is in line with the actual situation, indicating that the obtained abnormal detection results are authentic. From the mapping graph of the two datasets (Car-Hailing \& Bicycle-Sharing), The number of anomalies distributed near Xigu Park is more than other parks. Based on this result, we searched the Weibo dataset and found that there were many posts related to Xigu Park that day, and the content of the posts were mostly records of visiting Xigu Park in the morning. Through further analysis of the detection results of the two datasets, we can find more information. Finally, from the mapping graph of the three datasets (Car-Hailing \& Bicycle-Sharing \& Subway), it can be found that there are many abnormal points near Tianjin University, Nankai University and Tianjin No. 21 Middle School, but there is no such situation near Tianjin No. 1 Middle School and Tianjin No. 2 Middle School. In view of this phenomenon, we consult relevant data and draw the following conclusions: In this experiment, the date of the dataset is December 21, 2019, which is the first day of the National Unified Entrance examination for Master's graduates. Therefore, the phenomenon of anomaly gathering near the school on this day is closely related to the Unified National Graduate Entrance Examination. Most of the candidates who took the exams at Tianjin University and Nankai University were students of their own schools, and their travel methods were mainly walking and riding shared bicycles. The candidates who took the exam at Tianjin No. 21 Middle School were all off-campus candidates, and their travel methods were mostly Car-Hailing, So there are a lot of anomalies around these three sites. In contrast, Tianjin No. 1 Middle School and Tianjin No. 2 Middle School are not test sites, so there are no abnormal points around them. According to the above analysis, we can find that different numbers of datasets can obtain different interpretable results. The more datasets we have, the more anomalies we can get, which is exactly what \texttt{\textsc{FadMan}} aims to achieve.

\section{Conclusion}
In this paper, we study the problem of federated anomaly detection across multiple attributed networks and propose a federated learning solution, \texttt{\textsc{FadMan}}. Our algorithm first introduces the network alignment method to the anomaly subgraph detection across  multiple attributed networks, and protects private industry data. Our algorithm can be applied to detect anomalies on attributeless networks. A solid theoretical basis is developed for our algorithm.


\appendix

\bibliographystyle{plainnat}
\bibliography{references.bib}


\appendix

\section{Appendix}

\subsection{Related work}
Our work is related to anomaly detection and network alignment. Here, we briefly introduce the related work in these two aspects.

\textbf{Anomaly detection} 
Anomaly detection has always been the focus of attention. Point anomalies only assign outliers to nodes, which can be regarded as a binary 0/1 classification problem \citep{akoglu2015graph,wang2018anomaly,eswaran2018spotlight}, and there is no connection between the detected abnormal nodes. However, with the expansion of anomaly detection in the field of graphics and the needs of actual scenes, abnormal nodes usually need to be displayed as connected subgraphs. On the other hand, the discovery of anomalies is often inseparable from statistical data. Compared with traditional parameterized scanning statistics (Kulldorff statistical data \citep{kulldorff1997spatial}), nonparametric graph scan statistics (NPGS) can be applied to heterogeneous graph data because it is free of distribution assumption. Therefore, many NPGS-based abnormal connected subgraph detection algorithms were born in combination with actual scenarios and performance requirements. They can be divided into exact algorithms \citep{takahashi2008flexibly,speakman2015scalable,le2019probabilistic}  and approximate algorithms
\citep{chen2014non,wu2018nonparametric,speakman2013dynamic}. 
Among them, Wu.etal proved that anomaly subgraph detection is an NP-hard problem, and proposed TSPSD \citep{wu2018nonparametric} based on dynamic programming\citep{cadena2018graph,bhatia2020midas,cai2020structural,boniol2020series2graph}. The algorithm approximates the graph to a tree topology and can be used for the large-scale dataset. Most of these algorithms are used in a single network, and only a few are used in multi-layer network scenarios  \citep{bindu2017discovering,de2015ranking,tam2019anomaly,jie2020framework}, and they are all used to identify abnormal nodes, not subgraphs. 
To the best of our knowledge, the latest work used to detect multiple networks' anomaly subgraphs is ASD-FT \citep{Ying2020anomaly}. It is a method of detecting anomaly subgraphs of the graph based on anomaly features of another graph. It introduces the basic idea of network alignment to capture anomaly features' transmission by inferring the basic edges between multiple entity networks. However, it is suitable for two network scenarios, which is different from our work.

\textbf{Network alignment}
Network alignment is the basic problem of cross-network mining, and many papers have proposed solutions. Most of them are based on attributes and structure. Traditional methods  \citep{perito2011unique,vosecky2009user,nassar2018low,mathieson2019using} mostly use entity tag information to achieve alignment, such as user nicknames in social networks and entity names in knowledge graphs. Manually defining features is another method \citep{almishari2012exploring,ni2018network}. This method needs to carefully design features manually for specific problems, and it is not easy to migrate to other scenarios. Most of the above two types of methods only consider attribute information, while some methods consider both network structure and attribute information (COSNET \citep{zhang2015cosnet}, 
HYDRA \citep{liu2014hydra}, 
REGAL \citep{heimann2018regal}\citep{ye2019vectorized,huynh2019network,yan2021bright}). They want to complement the network structure and attribute information to achieve a better alignment effect. Moreover, there are many network alignment studies based on deep learning\citep{sun2019dna,wang2018deepmatching,meng2019deep}. In addition, because the attribute information may be falsely fabricated or lost or hidden due to privacy, there are many alignment  algorithms based only on structural information
 (BigAlign \citep{koutra2013big}, UMA \citep{zhang2015multiple}, IONE \citep{liu2016aligning}, CrossMNA \citep{chu2019cross}). Among them, UMA, REGAL, and CrossMNA can be applied to multiple network scenarios. Nevertheless, UMA and REGAL follow the assumption of topological consistency and cannot handle networks with different structures. However, CrossMNA does not follow topological consistency and can learn a common structure across network diversity. By integrating information from different networks, enhances the effect of embedding and effectively reduces space overhead, so it is suitable for large-scale multi-network scenarios.

Our work is based on the TSPSD and CrossMNA algorithms. Compared with the existing work, we are innovative and superior to the baselines in terms of efficiency and comprehensiveness.

\begin{figure}[t]
    \centering 
    \includegraphics[width=1\linewidth]{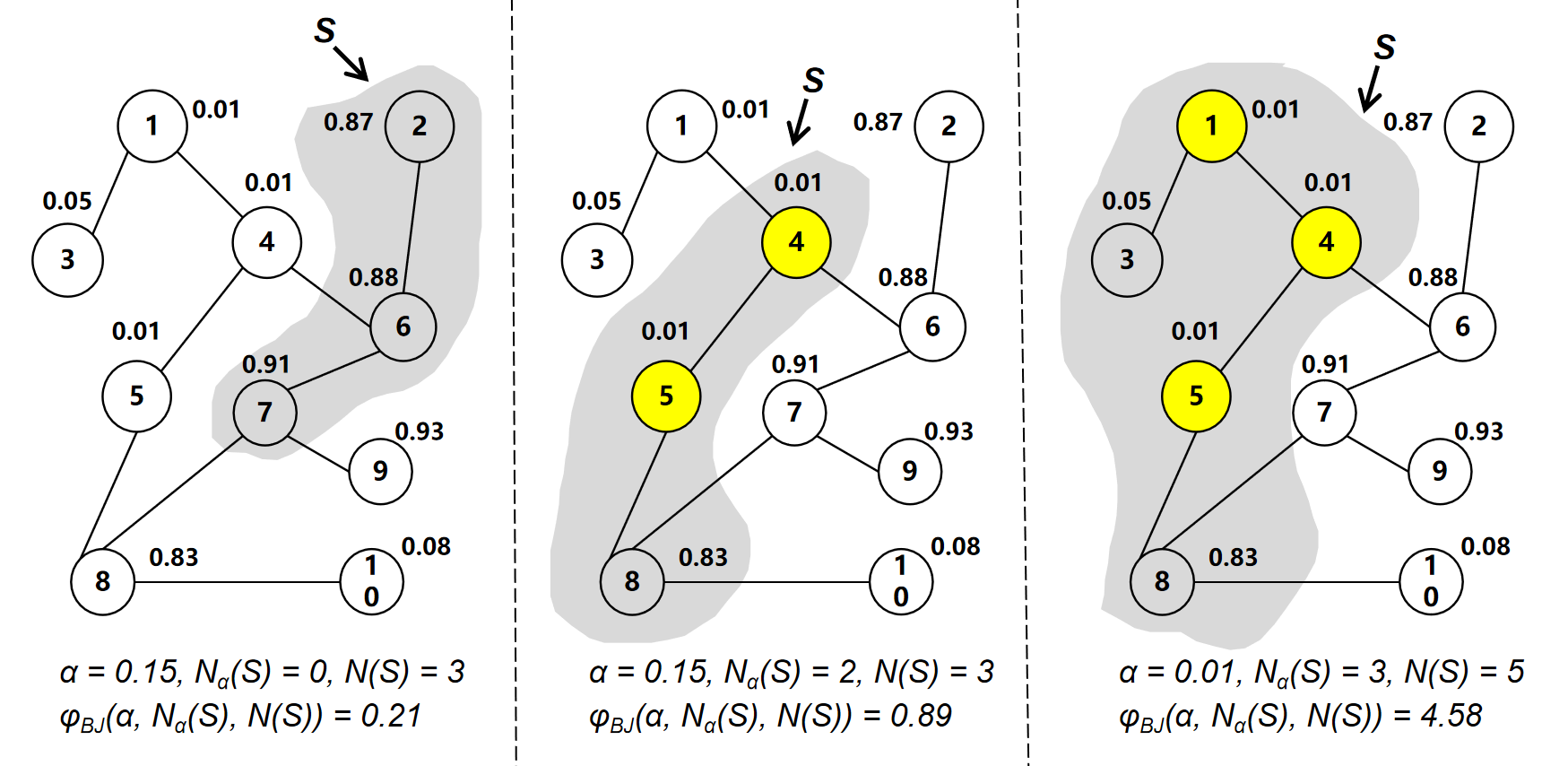}
    \caption{The BJ statistical scores for the three example subgraphs show that the score function increases with $N_{\alpha}(S)$ and decreases with $N(S)-N_{\alpha}(S)$ and $\alpha$. Yellow points are the vertices with p-values less than or equal to $\alpha$.
}
\label{fig:BJ}
\vspace{-0.3cm}
\end{figure}

\subsection{Two non-parametric graph scan statistics}
A data owner has an attributed network $G = (V, E, P)$, $V=\{v_1,\cdots,v_n\}$,  $E\subseteq V\times V$ are sets of vertices and edges in $G$, and $P \in \mathbb{R}^n$ is the specific anomaly feature set of $G$. $P$ is widely obtained by the mapping function $\mathbf{p}: V \rightarrow[0, 1]$ defines the empirical p-value corresponding to each node $v \in V$ \citep{wu2018nonparametric,chen2014non}, the smaller the p-value, the more abnormal the node.
Anomaly target detection can be represented as identifying subgraphs $S\subseteq G$  whose vertex set $V_{S}$ and edge set $E_{S}$ are subset of $V$, $E$.  

\textbf{Intersection of subgraphs} Subgraph $S_1\subseteq G$  whose vertex set $V_{S_1}$ and edge set $E_{S_1}$ are subset of $V$, $E$.
Subgraph $S_2\subseteq G$  whose vertex set $V_{S_2}$ and edge set $E_{S_2}$ are subset of $V$, $E$. $S$ is the intersection set of $S_1$ and $S_2$, where $V_{S}=V_{S_1} \cap V_{S_2}$, $E_{S}=E_{S_1} \cap E_{S_2}$.

\textbf{Union of subgraphs} Subgraph $S_1\subseteq G$  whose vertex set $V_{S_1}$ and edge set $E_{S_1}$ are subset of $V$, $E$.
Subgraph $S_2\subseteq G$  whose vertex set $V_{S_2}$ and edge set $E_{S_2}$ are subset of $V$, $E$. $S$ is the union set of $S_1$ and $S_2$, where $V_{S}=V_{S_1} \cup V_{S_2}$, $E_{S}=E_{S_1} \cup E_{S_2}$.

\textbf{Difference of subgraphs} Subgraph $S_1\subseteq G$  whose vertex set $V_{S_1}$ and edge set $E_{S_1}$ are subset of $V$, $E$.
Subgraph $S_2\subseteq G$  whose vertex set $V_{S_2}$ and edge set $E_{S_2}$ are subset of $V$, $E$. $S$ is the difference set of $S_1$ and $S_2$, where $V_{S}=V_{S_1} - V_{S_2}$, $E_{S}=E_{S_1} - E_{S_2}$.

\textbf{empirical p-value} The empirical p-value of node v is then defined as:
\begin{equation}
p(v)=\frac{1}{T}\sum\limits_{t=1}^{T}I(f(v^{(t)})\geq f(v)).
\end{equation}
where $f(v^{(t)})$ refers to the feature vector of the node $v$ at time $t$. The empirical value $p(v)$ defined above can
be interpreted as the proportion of historical observations $f(v^{(t)})$ when there was no event occurring with observed
values that are greater than or equal to the current observation $f(v)$.

\textbf{BJ statistics} 
    Berk-Jones (BJ) statistic is defined as follows~\citep{bj}:
    \begin{equation}\label{bj}
        \varphi _{BJ}(\alpha,N_\alpha(S),N(S))=N(S)\times KL(\frac{N_{\alpha}(S)}{N(S)},\alpha),
    \end{equation}
    where KL is Kullback-Liebler divergence between the observed and expected proportions of p-values less than $\alpha$, its formulation is presented as follows:
    \begin{equation}
        KL(a,b)=\left\{\begin{matrix}
    a\log(\frac{a}{b})+(1-a)\log(\frac{1-a}{1-b}) , & if \quad a \geq b \\ 
     0 , & if\quad  a < b
    \end{matrix}\right.
    \end{equation}
    The BJ statistic can be interpreted as the log-likelihood ratio statistic for testing whether the empirical p-value follows a uniform or piecewise constant distribution. Berk and Jones demonstrated that this statistic fulfills several optimality properties and has greater power than any weighted Kolmogorov statistic. We illustrate the BJ statistic in Figure\ref{fig:BJ}.
    	
\textbf{HC statistics} 
    Higher Criticism (HC) statistic is defined as follows~\citep{hc}:
    \begin{equation}\label{hc}
        \varphi_{HC}(\alpha,N_{\alpha}(S),N(S))=\frac{N_{\alpha}(S)-N(S)\alpha}{\sqrt{N(S)\alpha(1-\alpha)}}.
    \end{equation}
	The HC statistic can be interpreted as the log-likelihood ratio statistic for testing whether the empirical p-values follow a uniform or binomial distribution with $N(S)$ and $\alpha$.

\textbf{Anomalies in our experiments} In the traffic datasets, anomalous nodes refer to nodes with high number of trips, the more the number of trips, the more anomalous the node is. Through experiments on traffic datasets we can conclude that the more datasets we have, the more anomaly information we get. For example, anomaly detection on one dataset can only detect the node itself as a location with high traffic such as shopping malls, hospitals, etc. On December 21, 2019, for the location hospital, anomaly detection on a single dataset can detect it as anomaly. Anomaly detection on the shared bicycle dataset only cannot detect the anomaly node of Tianjin No. 21 High School, but by performing anomaly detection on three datasets, the location can be found to be anomalous because it is a weekend and school is not in session, students living nearby usually choose shared bicycles as transportation to school instead of online cars, since this day is the first day of the National Unified Entrance examination for Master's graduates, candidates usually choose convention cars as their transportation, so this abnormal node of Tianjin No. 21 High School is detected.

\subsection{Theoretical analysis}

We now analyze the increasing property in the objective when one iteration of the proposed algorithm is performed. The convergence of \texttt{\textsc{FadMan}} is proved based on several lemmas and the four properties of objective function, without the specific forms of $F$ and $Q$ functions. Each lemma is a building-block that describes that there exists an local optimal solution.
\begin{theorem}[Convergence and Optimality of \texttt{\textsc{FadMan}}]
\label{theorem:main1}
Within the pre-aligned domain (each data owner integrates the alignment probability information of all node pairs between its private network and the public network), given the parameter $\alpha$ and $\sigma$ settings,  Algorithm 1 converges to the optimal solution of the problem. 
\end{theorem}
\begin{proof}
At the $t$-th iteration, we have the public anomaly $U^t$. By the step 5 at Algorithm 1, we can obtain the local optimum solution $\{S_i^t\}$.
\begin{displaymath}
\sum_{i \in [N]} F_\alpha^i(S_i^t) + Q_\sigma^i(S_i^t, U^t)
\end{displaymath}
By the step 6 at Algorithm 1, we can obtain the new local optimum public anomaly $\{U_i^*\}$. We must upload $U_i^*$ to the server if $Q_\sigma^i(S_i^t, U^*_i) > Q_\sigma^i(S_i^t, U^t)$, otherwise upload $U^t$ to the server for the $i$-th owner. By the step 9 at Algorithm 1 and the lemma~\ref{lmthm1}, we can obtain an optimum public anomaly $U^*$ for the set $\{S_i^t\}$. Thus we have $\sum_{i \in [N]}Q_\sigma^i(S_i^t, U^*) > \sum_{i \in [N]}Q_\sigma^i(S_i^t, U^t)$. We consider $U^*$ as the new public anomaly $U^{t+1}$, and we must have the following inequality:
\begin{displaymath}
\sum_{i \in [N]} F_\alpha^i(S_i^t) + Q_\sigma^i(S_i^t, U^{t+1}) > \sum_{i \in [N]} F_\alpha^i(S_i^t) + Q_\sigma^i(S_i^t, U^t)
\end{displaymath}
As $S_i^{t+1}$ is the local maximum solution at the step 5 at Algorithm 1, we can obtain the new inequailty $\sum_{i \in [N]} F_\alpha^i(S_i^{t+1}) + Q_\sigma^i(S_i^{t+1}, U^{t+1}) > \sum_{i \in [N]} F_\alpha^i(S_i^t) + Q_\sigma^i(S_i^t, U^t)$. Algorithm 1 is repeated until the public anomaly is not changed. We obtain the optimum solution. We proved the theorem.
\end{proof}
Though Theorem~\ref{theorem:main1} looks straightforward to a network-based algorithm framework for federated anomaly detection, the problems at steps (c1) and (c2) can be addressed with the previous methods~\cite{Ying2020anomaly,wu2018nonparametric} support the optimal solutions. At the step (c1), given the public anomaly $U$, the optimal private anomaly $S_i^*$ is obtained. We are target for the new improved public anomaly $U$ at the next iteration. 

At the $t$-th iteration, for the server side, we can observe the public anomalies $\{U_i^*\}$ and the function values of $F$ and $Q$. We assume that there exists a public anomaly subgraph $\underline{U^* \subseteq G_0}$ maximizing $\sum_{i \in [N]} Q_\sigma^i(S_i^*, U^*)$ under the fixed $S_i^*$. The value of $F_\alpha^i(S_i^*)$ will not change with the update $U$.

\begin{lemma}[$Q$ upper bound]
We must have the upper bound of the optimal $Q$ under the fixed $S_i^*$.
\begin{equation}
\sum_{i=1}^N   Q_\sigma^i(S_i^*, U_i^*)
 \geq 
 \sum_{i=1}^N  Q_\sigma^i(S_i^*, U^*).
\end{equation}
\label{lmup1}
\end{lemma}
\begin{proof}
At the step 6 in Algorithm 1, we can observe that $U_i^*$ maximizes the $Q_\sigma^i(S_i^*, U_i^*)$ with the fixed $S_i^*$. Thus for the i-th data owner, given an public anomaly $U^*$, we must have
\begin{displaymath}
Q_\sigma^i(S_i^*, U_i^*) \geq Q_\sigma^i(S_i^*, U^*)
\end{displaymath}
For all of the local data owners, the inequality $\sum_{i=1}^N   Q_\sigma^i(S_i^*, U_i^*)
 \geq 
 \sum_{i=1}^N  Q_\sigma^i(S_i^*, U^*)$ must satisfied. We proved this lemma for $Q$ upper bound.
\end{proof}
\begin{lemma}[$Q$ lower bound]
A public anomaly $\underline{U^o \subseteq \bigcup_{i \in [N]} U_i^*}$ maximizes $\sum_{i \in [N]} Q_\sigma^i(S_i^*, U^o)$. We must have the lower bound of the optimal $Q$ under the fixed $S_i^*$.
\begin{equation}
 \sum_{i=1}^N   Q_\sigma^i(S_i^*, U^*)
 \geq 
 \sum_{i=1}^N   Q_\sigma^i(S_i^*, U^o).
 \label{eq:lower1}
\end{equation}
\label{lmlow1}
\end{lemma}
\begin{proof}
Under the fixed $S_i^*$, we have already assumed that the public anomaly subgraph $\underline{U^* \subseteq G_0}$ maximizing $\sum_{i \in [N]} Q_\sigma^i(S_i^*, U^*)$.
Thus for a public anomaly subgraph $\underline{U^o \subseteq \bigcup_{i \in [N]} U_i^*}$, we must have the inequality~(\ref{eq:lower1}). This proof is proved with the intuition ways.
\end{proof}
\begin{lemma}[$Q$ optimum solution]
The solution $U^o = U^*$ maximizes $Q$ under the fixed $S_i^*$ over $G_0$. The optimum solution must exist between
\begin{equation}
    \min_{U \subseteq \bigcup_{i \in [N]} U_i^*} \sum_{i=1}^N   \Big[Q_\sigma^i(S_i^*, U_i^*) -   Q_\sigma^i(S_i^*, U)\Big]
    \label{eqopt1}
\end{equation}
\label{lmopt1}
\end{lemma}
\begin{proof}
Proofs of the lemma consist of two phases. At the first phase, we prove that there exists a public anomaly $U \subseteq G_0$ which minimize the gap between the $Q$ upper bound and lower bound. According to Lemma~\ref{lmup1} and \ref{lmlow1}, we can observe that the optimum public anomaly $U$ must minimize the following problem:
\begin{displaymath}
\min_{U \subseteq G_0} \sum_{i=1}^N   \Big[Q_\sigma^i(S_i^*, U_i^*) -   Q_\sigma^i(S_i^*, U)\Big]
\end{displaymath}
Next, for the phase two, we are target to prove $U=U^o = U^*$. We first assume $U = U^o \cup U^-$, where $U^o \subseteq \bigcup_{i \in [N]} U_i^*$ and $U^- \subseteq G_0 - \bigcup_{i \in [N]} U_i^*$. Then, we have the optimum solution as follows:
\begin{displaymath}
\sum_{i=1}^N   \Big[Q_\sigma^i(S_i^*, U_i^*) -   Q_\sigma^i(S_i^*, U^o \cup U^-)\Big]
\end{displaymath}
which is equal to maximize $\sum_{i\in [N]}Q_\sigma^i(S_i^*, U^o \cup U^-)$. By the properties (P3) and (P4), we aim to prove that the vertices of $U^-$ must be aligned on subset of $\{S_i^*\}$. We first assume there exist a vertex $v \in V_{U^-}$, however, the vertex $v$ is not aligned with any vertices of $\{S_i^*\}$. Then, we must have the inequality $\sum_{i\in [N]}Q_\sigma^i(S_i^*, U^o \cup U^- - \{v\}) > \sum_{i\in [N]}Q_\sigma^i(S_i^*, U^o \cup U^-)$ by the property (P4), and the result contradicts with the maximum solution $U^o \cup U^-$. Thus all the vertices of $U^-$ must be aligned on subset of $\{S_i^*\}$.

We assume that the public anomaly $U^-$ is not empty. Without loss of generality, we assume there exist a vertex $v \in V_{U^-}$ which is aligned on the $i$-th local data owner $S_i^*$. We know that $v \notin V_{U^*_i}$. We must have the inequality $Q_\sigma^i(S_i^*, U^*_i \cup \{v\}) > Q_\sigma^i(S_i^*, U^*_i)$ by the property (P3), and the result contradicts with the maximum solution $U_i^*$. Thus the public anomaly $U^-$ is empty. We have $U=U^o$, and prove the lemma that $U^o$ is the optimum solution. 
\end{proof}
\begin{lemma}[Theorem 1, from~\cite{donahue2021optimality}]
  An optimal partition $\Pi$ can be created for a set of local data owners $[N]$. First, start with every owner maximizing the local public anomaly $U_i^*$. Then, group the owners together in ascending order of node size, halting at the first owner would increase its error $(\sum_{k \in C} |V_{U^*_k}|-|V_{U^*_i}|)/(|V_{U^*_i}| \sum_{k \in C} |V_{U^*_k}|)  Q_\sigma^i(S_i^*, U_i^*)$ by joining the coalition $C$. The resulting partition $\Pi$ is optimal.
\label{lmthm1}
\end{lemma}
\begin{proof}
The public anomaly $U_i^*$ is the local maximum solution. The $Q$ upper bound is derived intuitively from summing local maximum solutions is greater than the global maximum solution.
For the $Q$ lower bound, the maximum solution over the local set is less than or equal the maximum solution over the global set.
We can observe that the maximum solution must exist between the upper bound and lower bound, and minimize the gap between the two bounds. Note that when the gap is 0, the optimum solution is equal to the solution of lower bound. All of  local solutions $U^*_i$ are the same.

However, in the problem~(\ref{eqopt1}), $Q_\sigma^i(S_i^*, U)$ can not be computed at the server because the forms of function $Q$ and private anomalies $\{S_i^*\}$ are not accessed for the data protection rules. We employ the framework for optimality and stability in federated learning to address the problem~(\ref{eqopt1})~\cite{donahue2021optimality}. We define a coalition $C \subseteq [N]$ is a subset of local data owners, and all owners $[N]$ are partitioned into coalition set $\{C\}$. We consider the gap between two bounds~(\ref{eqopt1}) as errors across owners. By the properties P3 and P4, we know that the value of $Q_\sigma^i(S_i^*, U)$ changed with the size of $U$ for the fixed $S_i^*$. Given a partition $\Pi$ of $[N]$, we define its error below:
\begin{equation}
  f_{\{U_i^*\}}(\Pi) = \sum_{C \in \prod} \sum_{i \in C} |V_{U_i^*}| \cdot \bigg(\frac{\sum_{k \in C} |V_{U^*_k}|-|V_{U^*_i}|}{|V_{U^*_i}|\cdot \sum_{k \in C} |V_{U^*_k}|}\bigg)  Q_\sigma^i(S_i^*, U_i^*)
  \label{eq:ferror1}
\end{equation}
\end{proof}
By Lemma~\ref{lmthm1}, we select a coalition $C \in \Pi$ with the minimum error, and achieve the desired public anomaly $U=\bigcup_{i \in C} U_{i}^*$. The public anomaly $U$ is the optimal solution for $\sum_{i\in [N]}Q_\sigma^i(S_i^*,\cdot)$.

By the theorem, our algorithm guarantees on detecting the most anomaly subgraphs on multiple private attributed networks. \texttt{\textsc{FadMan}}'s time complexity is \textit{$O(kN|V|^2)$}, where $k$ is the number of iterations, $N$ is the number of networks, $V$ is the number of nodes of the network. In practice, the implementations of $F_\alpha$ and $Q_\sigma$ are normalized because the function value ranges are different.

\subsection{Datasets}
We constructed two multi-network scenarios based on the following five real datasets: Computer network, Car-hailing \& Bicycle-sharing \& Subway (Metro) \& POI (Point of Interest). POI is the groud truth dataset.

(1) \textbf{Computer network:} An Internet company provided browsing logs from the *\textbf{$\emph{edu.cn}$} websites, which involved a total of $996$ websites, $131,205$ IPs, and the time range was from May 31, 2014 to May 31, 2015, with a total of $3,978,073$ logs. In this time range, for a certain website/IP on the t-th day, we take the number of logs related to that website/IP on the t-th day as the observed value $c_t$. By comparing $c_t$ with the daily $c_i, (i<t)$ of the website/IP before the t-th day, the empirical p-value of the website/IP on the t-th day is obtained. Therefore, for all websites/IPs involved in the daily log, we have their corresponding p-value snapshots. We divided these logs into six parts, each containing two months of data, and constructed six computer networks based on this.
(2) \textbf{Car-hailing:} We use the online car-hailing dataset provided by a transportation company. This dataset was collected $58,674$ online car-hailing orders in Tianjin on December 20, 2019. Each record contains the time of the order and the start and end of the itinerary. We processed the dataset as follows: First, we divide the dataset into 24 parts according to the time (hours). For hourly data, we used the start point and endpoint of the itinerary as a node in the network and regarded the start point and endpoint of an itinerary is connected so that we can obtain the corresponding itinerary network of this hour. For the node $v$ in the itinerary network of the hour $t$, we regard the number of times the node $v$ is used as the start point or the endpoint in the hour $t$ as the observed value $c_t$, and regard observed values $c_i, (i \neq t)$ of $v$ in other hours as compared values, the empirical p-value of $v$ at hour $t$ can be obtained. As a result, we obtained 24 car-hailing itinerary networks.
(3) \textbf{Bicycle-sharing: } The dataset is also provided by a transportation company. It collected 229,814 bicycle-sharing orders in Tianjin on December 20, 2019. Each order contains time, start location, and end location. We perform the same processing on this dataset as the car-hailing dataset.
(4) \textbf{Subway: }This dataset collected subway traffic data in Tianjin on December 20, 2019. The data recorded the hourly passenger flow of 143 subway stations in Tianjin that day, totaling 3,432. According to the Tianjin subway route, connections are established between interconnected subway stations to form a subway network with 143 nodes and 153 edges. For the subway station $s$ at hour $t$, we regard the passenger flow of $s$ in that hour as the observed value $c_t$. Next, perform the same processing as the car-hailing dataset. As a result, we obtained 24 subway networks.
(5) \textbf{POI: } This dataset comes from Baidu Maps, recording the information of 242,189 interest points in Tianjin. Each message contains the name, location, telephone number, longitude, latitude, id, and keyword. This dataset can map the longitude and latitude information of the Car-hailing, Bicycle-sharing, and Subway datasets to specific locations, thereby obtaining more realistic results.

\begin{figure}[t]
    \centering 
    \includegraphics[width=1\linewidth]{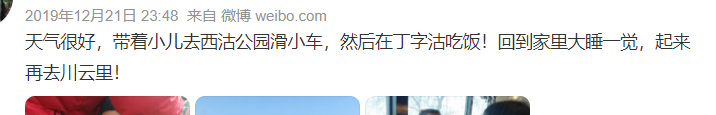}
    \vspace{-0.3cm}
    \caption{The translation of the weibo: Beautiful weather, took my son to Xigu Park for a little scoot, then dinner at Dingzigu! Came home for a big nap, got up and went to Chuangyunli again!
}
\label{fig:weibo1}
\end{figure}

\begin{figure}[t]
    \centering 
    \includegraphics[width=1\linewidth]{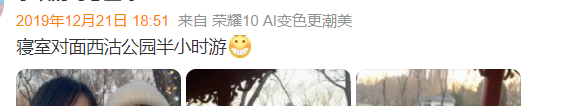}
    \caption{The translation of the weibo: I spent half an hour playing in the Xigu Park opposite my dormitory.
}
 \label{fig:weibo2}
 \vspace{-0.3cm}
\end{figure}

\begin{figure}[t]
    \centering 
    \includegraphics[width=1\linewidth]{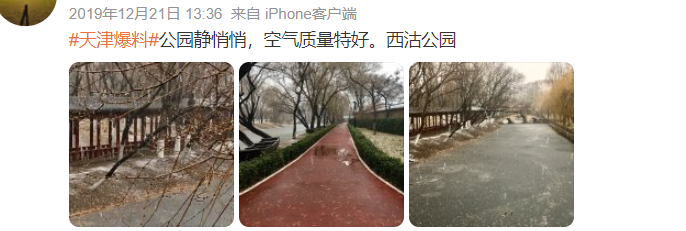}
    \caption{The translation of the weibo: $\#$Tianjin Explosion$\#$ The park is very quiet and the air quality is exceptionally good. Xigu Park.
}
\label{fig:weibo3}
\vspace{-0.3cm}
\end{figure}

\begin{figure}[t]
    \centering 
    \includegraphics[width=1\linewidth]{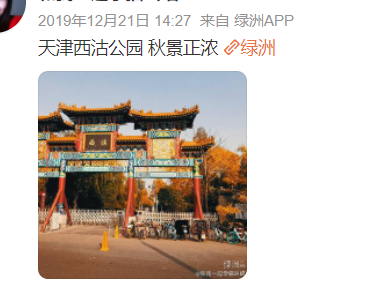}
    \caption{The translation of the weibo: Autumn is in full swing in Tianjin Xigu Park.
}
\label{fig:weibo4}
\vspace{-0.3cm}
\end{figure}

\subsection{Metrics}
\textbf{Metrics for evaluating algorithms' ability to detect correlated anomalies on multiple attributed networks:} The results of FADMAN are anomaly subgraphs. The anomaly subgraph contains only a few normal nodes, most of which are abnormal. For example, if an IP has a record of attacking a website, it is identified as an abnormal node. If it is in the algorithm result, it is a true negative; otherwise, it is a false positive. For a normal node that has not attacked the website, if it is in the result, it is a false negative; otherwise, it is a true positive. We compare the results with ground truth to get these metrics.
\textbf{Metrics for evaluating the algorithms' ability to detect anomalies on an attributeless network:} We regard $G_6$ as an attributeless network by setting p-values of its nodes to $1$. We obtain the correlated anomaly subgraphs of the first five networks and their public anomaly subgraph, then align the public anomaly subgraph to $G_6$ to obtain the anomaly subgraph of $G_6$. Anchor\_Count indicates the number of real abnormal anchor links between the public anomaly subgraph and $G_6$ detected by the algorithm we proposed.

\subsection{Weibo images in the traffic datasets case study}
The Weibo images mentioned in the traffic dataset case are shown in Figure\ref{fig:weibo1}, \ref{fig:weibo2}, \ref{fig:weibo3}, and \ref{fig:weibo4}.

\begin{figure*}[t]
\setlength{\abovecaptionskip}{0.1cm}
\setlength{\belowcaptionskip}{-0.1cm}
  \centering
  \includegraphics[width=13.5cm]{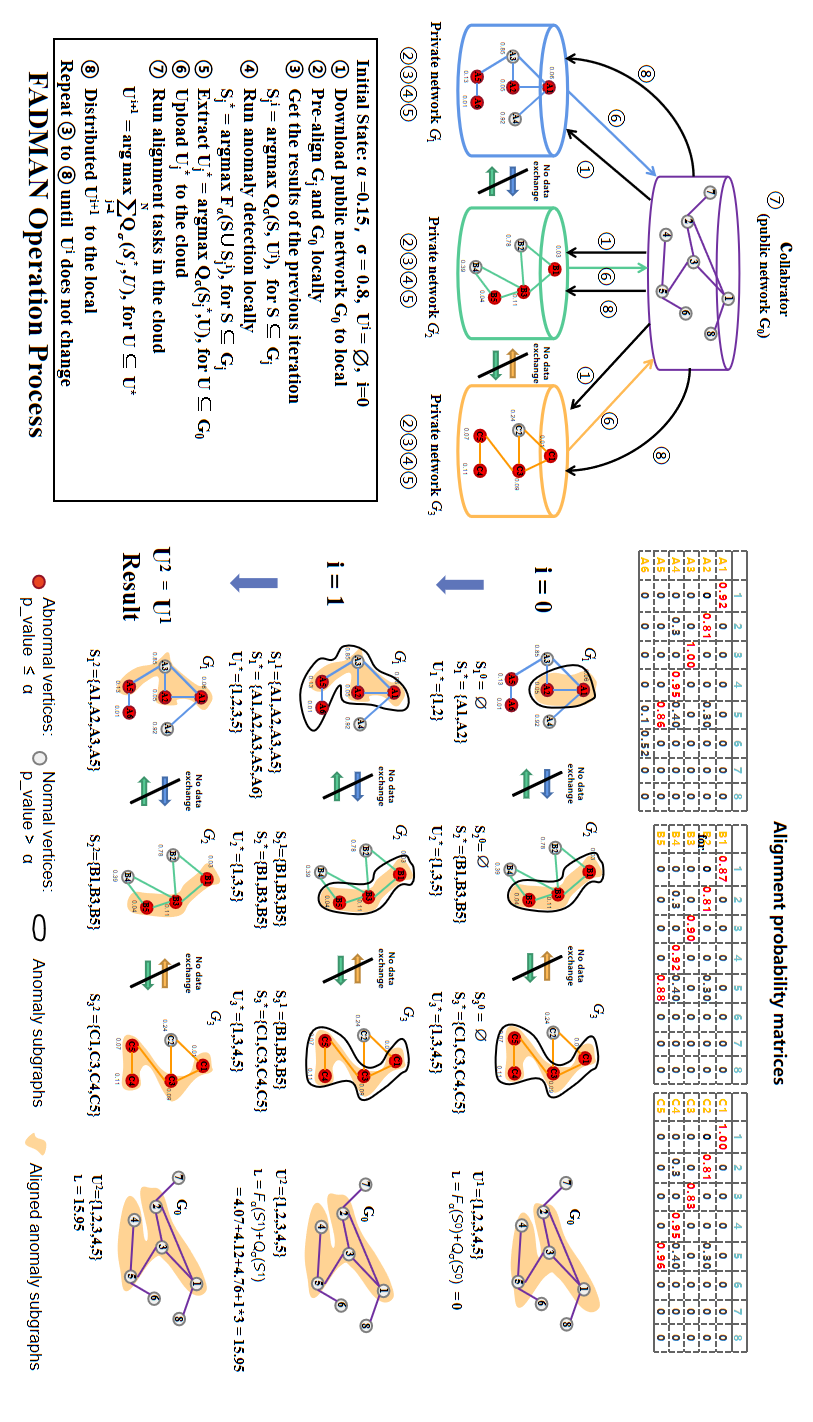}
  \caption{\textbf{Illustration of our algorithm.}
  The subgraphs within solid black line freeform shapes are the largest anomaly subgraphs (i.e., $\max F_\alpha$), and the yellow shaded subgraphs represent the aligned subgraphs (i.e., $\max Q_\sigma$). The vertex values are empirical p-values. }
  \label{fig:FADMAN}
  \vspace{-5mm}
\end{figure*}

\end{document}